\documentclass[preprint,12pt]{colt2026} %
\jmlrproceedings{arXiv}{arXiv preprint}

\usepackage{microtype}
\usepackage{graphicx}
\usepackage{booktabs} 

\usepackage{hyperref}

\usepackage{mathtools}
\usepackage{xspace}
\usepackage{xcolor}
\usepackage{url}
\usepackage{multicol}
\usepackage{multirow}
\usepackage{cleveref}

\SetAlgoLined
\LinesNumbered
\SetKwInOut{KwIn}{Input}
\SetKwInOut{KwOut}{Output}
\SetKw{KwTo}{to}

\crefname{theorem}{theorem}{theorems}\Crefname{theorem}{Theorem}{Theorems}
\crefname{assumption}{assumption}{assumptions}\Crefname{assumption}{Assumption}{Assumptions}
\crefname{lemma}{lemma}{lemmas}\Crefname{lemma}{Lemma}{Lemmas}
\crefname{proposition}{proposition}{propositions}\Crefname{proposition}{Proposition}{Propositions}
\crefname{corollary}{corollary}{corollaries}\Crefname{corollary}{Corollary}{Corollaries}
\crefname{definition}{definition}{definitions}\Crefname{definition}{Definition}{Definitions}
\crefname{remark}{remark}{remarks}\Crefname{remark}{Remark}{Remarks}
\crefname{example}{example}{examples}\Crefname{example}{Example}{Examples}
\crefname{algocf}{algorithm}{algorithms}\Crefname{algocf}{Algorithm}{Algorithms}

\DeclareMathOperator*{\argmax}{arg\,max}

\newcommand{\items}{D}             %
\newcommand{\nitems}{n}            %
\newcommand{\val}{v}               %
\newcommand{\TopK}{T_k^\star}      %
\newcommand{\thresh}{t^\star}      %
\newcommand{\gap}{\Delta}          %
\newcommand{\weak}{\mathsf{W}}     %
\newcommand{\strong}{\mathsf{S}}   %
\newcommand{\weakval}{\tilde v}    %
\newcommand{\weakmean}{\mu}        %

\newcommand{\CIlo}{L}
\newcommand{\CIhi}{U}
\newcommand{\CIrad}[1]{\varepsilon(#1)}
\newcommand{\CImax}{\varepsilon_{\max}}
\newcommand{\Conf}{\mathcal{E}}    %

\newcommand{\Prb}{\mathbb{P}}
\newcommand{\Exp}{\mathbb{E}}

\newcommand{\mnearties}[1]{m(#1)}

\newcommand{\algbase}{\textsc{STC}\xspace}     %
\newcommand{\ACE}{\textsc{ACE}\xspace}         %
\newcommand{\ACEW}{\textsc{ACE-W}\xspace}      %

\newcommand{\topkof}[2]{\ensuremath{\mathrm{top}_{#2}(#1)}} %

\newcommand{\Amb}{A}
\newcommand{\mmin}{w_{\min}}
\newcommand{\mmax}{w_{\max}} %

\usepackage{xcolor}

\title[Top-k on a Budget: Adaptive Ranking with Weak and Strong Oracles]{Top-k on a Budget: Adaptive Ranking with Weak and Strong Oracles}
\usepackage{times}
\coltauthor{%
 \Name{Lutz Oettershagen} \Email{lutz.oettershagen@liverpool.ac.uk}\\
 \addr University of Liverpool, UK
}

\begin{document}

\maketitle

\begin{abstract}%
Identifying the top-$k$ items is fundamental but often prohibitive when exact valuations are expensive. We study a two-oracle setting with a fast, noisy weak oracle and a scarce, high-fidelity strong oracle (e.g., human expert verification or expensive simulation). 
We first analyze a simple screen-then-certify baseline (\algbase) and prove it makes at most $m(4\varepsilon_{\max})$ strong calls given jointly valid weak confidence intervals with maximum radius $\varepsilon_{\max}$, where $m(\cdot)$ denotes the near-tie mass around the top-$k$ threshold. We establish a conditional lower bound of $\Omega(m(\varepsilon_{\max}))$ for any algorithm given the same weak uncertainty.
Our main contribution is \ACE, an adaptive certification algorithm that focuses strong queries on critical boundary items, achieving the same $O(m(4\varepsilon_{\max}))$ bound while reducing strong calls in practice. We then introduce \ACEW, a fully adaptive two-phase method that allocates weak budget adaptively before running \ACE, further reducing strong costs. 
\end{abstract}


\section{Introduction}

Identifying the top-$k$ items from a collection is a central task across machine learning, data mining, and information retrieval.
Applications span recommender systems \citep{luo2025recranker}, data valuation \citep{ghorbani2019data}, explainability via feature attribution \citep{lundberg2017unified}, influence maximization in networks \citep{kempe2003maximizing}, and scientific discovery such as drug candidate selection.
The common challenge is that the valuation function determining item quality is often \emph{expensive} to compute, whether due to repeated model retraining, costly simulations, or large-scale graph computations \citep{ilyas2008survey}.
This challenge is often further amplified by high-dimensional representations~\citep{ren2021survey}.
For example, computing similarities in learned embedding spaces can incur substantial computational cost~\citep{johnson2019billion}, making brute-force scoring impractical~\citep{khattab2020colbert}.
These trends highlight a structural bottleneck: while many applications require identifying only the top-$k$ items, obtaining exact scores for all candidates is often prohibitively expensive.
A natural solution is to exploit multiple information sources of varying cost and fidelity.
We formalize this common structure through a simple abstraction, assuming access to two types of oracles:
\begin{itemize}
	\item a \emph{strong oracle}, which provides (near-)exact valuations but is scarce and costly
	(e.g., human expert verification, full retraining, or high-fidelity simulation);
	\item a \emph{weak oracle}, which yields fast but noisy estimates
	(e.g., low-budget Monte Carlo, lightweight surrogates, sketches, or low-fidelity simulations).
\end{itemize}
Weak oracles provide broad but imperfect coverage, while strong oracles are reliable but budget-limited.
This raises the central question we address:
\begin{quote}
	\emph{How can we combine weak and strong oracles to identify the exact top-$k$ set with high probability, while minimizing the number of expensive/scarce strong oracle calls?}
\end{quote}

We study a {two-resource} setting: weak queries are available in bulk 
while strong queries are the scarce resource we seek to minimize.
This reflects applications where strong evaluation is qualitatively constrained (e.g., expert time or expensive compute),
so it is not meaningful to trade weak and strong calls using a single exchange rate.

In many applications, identifying the top-$k$ items requires access to scarce human expertise.
For example, in document screening for legal compliance, medical triage, or content moderation,
automated signals (e.g., heuristics or lightweight classifiers) can quickly provide noisy relevance scores,
while borderline cases must be reviewed by a human expert~\citep{okati2021differentiable,zhang2015active,mosqueira2023human,raghu2019algorithmic}.
In the top-$k$ ranking setting, the weak oracle corresponds to automated screening, while the strong oracle corresponds to expert judgment,
and the objective is to certify the exact top-$k$ items while minimizing expert evaluations.

The same two-oracle structure arises in fully automated ML pipelines.
A central example is \emph{data valuation}, where the goal is to identify the most valuable training examples for a predictive task.
Shapley values~\citep{ghorbani2019data,jia2019towards} provide a principled measure of data importance, but exact computation is infeasible.
Monte Carlo (MC) approximation yields fast but noisy estimates, while high-budget MC or exact solutions are accurate but expensive.
Here, low-budget MC estimates act as the weak oracle, and high-budget estimates as the strong oracle.
Similar two-oracle structures arise in feature attribution, influence maximization~\citep{kempe2003maximizing}, and hyperparameter screening, where cheap proxies must be validated by expensive evaluations.

\paragraph{Our contributions.}
We {initiate the study} of PAC certification of the exact top-$k$ set with a weak and a strong oracle, formalizing this asymmetric two-oracle model where weak queries provide noisy estimates in bulk and strong queries give exact values but are scarce, with the objective of minimizing strong oracle calls.

\begin{enumerate}
	\item We formalize this two-oracle top-$k$ certification problem and introduce a one-shot \emph{screen-then-certify} baseline (\algbase), proving it makes at most $m(4\varepsilon_{\max})$ strong calls given jointly valid weak CIs, where $m(\cdot)$ denotes the near-tie mass. We establish a conditional lower bound of $\Omega(m(\varepsilon_{\max}))$ for any algorithm given the same weak uncertainty.

	\item We propose \ACE{} (\emph{\textbf{A}daptive \textbf{C}ertification of the \textbf{E}xact top-$k$}), the first adaptive algorithm for this setting, which focuses strong queries on critical boundary items and achieves the same $O\left(\mnearties{4\CImax}\right)$ upper bound while substantially reducing strong calls in practice.
	
	\item We extend this with \ACEW, which adaptively allocates weak budget to ambiguous items before running \ACE, shrinking $\CImax$ and further reducing strong-call costs.
	
	\item Experiments validate our theoretical scaling and demonstrate practical effectiveness.
\end{enumerate}

\section{Related Work}

Identifying top-$k$ items under evaluation budgets spans databases, graph analytics, bandit learning, and multi-fidelity optimization. We focus on a gap not covered by these strands: \emph{PAC certification of the exact top-$k$ set} when only a \emph{noisy weak} oracle and an \emph{expensive and scarce exact} strong oracle are available, with the objective of \emph{minimizing strong oracle calls}. The two-oracle paradigm has also proven effective in other domains, such as correlation clustering~\citep{silwal2023kwikbucks}.

In databases, the Threshold Algorithm (TA)~\citep{FAGIN2003614} and variants combine multiple sorted access paths for exact top-$k$ retrieval. Related ideas appear in graph analytics, where sampling or sketching pre-ranks candidates before exact evaluation~\citep{okamoto2008ranking,bergamini2017closeness,rionato2016betweenness}. These methods embody \emph{screen-then-certify}: inexpensive but imperfect scores identify a candidate set, then exact evaluations certify it. Our one-shot baseline extends this principle to settings where weak scores are \emph{stochastic and unsorted}, requiring PAC guarantees rather than deterministic thresholds.
In ML, valuation functions are often too costly for brute-force scoring: data valuation via Shapley values~\citep{ghorbani2019data,jia2019towards}, feature attribution~\citep{lundberg2017unified,covert2021improving}, and influence processes in networks~\citep{kempe2003maximizing}. Usage of noisy proxies (e.g., gradient norms, truncated influence, low-budget KernelSHAP) is common but often lacks guarantees that the resulting top-$k$ matches the true one. Recent work has highlighted the need for statistical guarantees in these settings, for instance by providing PAC guarantees for top-k feature rankings~\citep{goldwasser2024statistical}. Our framework provides a general method for such certification by combining weak and strong oracles with jointly valid confidence intervals.

Bandits and best-arm identification (BAI) methods~\citep{kalyanakrishnan2012pac} and extensions to top-$k$~\citep{kalyanakrishnan2010efficient,kaufmann2013information} and combinatorial pure exploration~\citep{chen2014combinatorial} adaptively sample \emph{a single stochastic oracle} and offer PAC guarantees. We borrow LUCB-style ideas and uncertainty-based approaches~\citep{kalyanakrishnan2012pac} for deciding \emph{which item} to refine, but our setting differs in two key ways: (i) certification proceeds after a \emph{separate weak phase} that yields \emph{jointly valid} per-item CIs; and (ii) our cost objective counts \emph{only strong calls}, with weak pulls front-loaded (or targeted in an optional ACE-W variant).
Multi-fidelity bandits and Bayesian optimization (BO)~\citep{kandasamy2016multi,kandasamy2017multi,poiani2022multi} use cheap low-fidelity evaluations to guide expensive high-fidelity queries.\looseness=-1

Our setting differs fundamentally from multi-fidelity BAI/BO: (i) we certify the \emph{exact top-$k$ set} via \emph{deterministic interval dominance} after a weak phase with \emph{jointly valid} CIs; (ii) our complexity and lower bounds are stated in \emph{strong-call units}, conditional on a weak budget; and (iii) the strong oracle is treated as (near-) exact while weak observations may be noisy and are only used to screen. Prior work typically optimizes \emph{total} pulls/regret for top-1 or treats fidelities symmetrically as stochastic pulls. Our results give \emph{instance-dependent} upper bounds matched by packing lower bounds in strong-call units.

Recent works have begun to formalize ranking with multiple oracles. \citeauthor{agarwal2022pac} study PAC top-$k$ identification with a pairwise comparison oracle under stochastic transitivity~\citep{agarwal2022pac}, a different setting from our scalar-value oracles. Closer to our work, \citeauthor{jin2025ranking}~\citep{jin2025ranking} study full-ranking recovery from pairwise comparisons provided by multiple heterogeneous oracles. They give bi-level algorithms and instance-dependent bounds in terms of pairwise hardness. In contrast, we address \emph{exact top-$k$ certification} with a two-fidelity model, weak scalar evaluations with variance and an exact strong oracle. The settings are complementary: dueling-style multi-oracle aggregation versus scalar-value certification under a high-fidelity budget.

\section{Problem Setting}
\label{sec:setting}

We are given items $\items=\{x_1,\dots,x_{\nitems}\}$ with unknown ground-truth valuation $\val:\items\to[0,1]$.
Let $\TopK\subseteq\items$ denote the $k$ items with largest $\val(x)$. We assume items are totally ordered by $v$ (ties are broken deterministically by a fixed index order). Let $v_{(k)}$ denote the $k$-th order statistic and $\thresh:=v_{(k)}$.
Define the (unknown) top-$k$ gap $\gap := v_{(k)} - v_{(k+1)}$ (for $k<\nitems$).

Our goal is to output $\widehat T_k=\TopK$ with high probability. Specifically, an algorithm is $(\delta,k)$-\emph{PAC} if it returns $\widehat T_k$ such that
\[
\Prb(\widehat T_k = T_k^\star)\ge 1-\delta.
\]

\paragraph{The two oracles setting.}
As mentioned before, we assume access to two oracles.
A query to the \emph{weak} oracle $\weak$ on $x$ yields $\weakval(x)\in[0,1]$ with mean
$\weakmean(x)=\Exp[\weakval(x)]=\val(x)$ and bounded variance,
\[
\mathrm{Var}(\weakval(x))\le \sigma^2(x)<\infty.
\]
Repeated weak queries are i.i.d.

A query to the \emph{strong} oracle $\strong$ returns $\val(x)$ exactly or with significantly higher precision than the weak oracle. 
In the following, we assume $\strong$ to be exact but proofs extend to noisy strong oracle settings by reserving a small part of~$\delta$, i.e., we allocate failure probability as $\delta=\delta_{\text{weak}}+\delta_{\text{strong}}$.
In the following, we assume $\delta_{\text{strong}}=0$ (exact strong oracle), hence $\delta_{\text{weak}}=\delta$.

Our objective is to minimize the number of strong oracle queries. Weak oracle queries are either (i) uniformly allocated $N$ per item or (ii) allocated from a total budget $B$.

\paragraph{No structural assumptions.}
Our strong-call complexity is governed by the \emph{near-tie mass} $m(\cdot)$ defined below.
In general, $m(\cdot)$ can be as large as $n$; our bounds and lower bounds are stated directly in terms of $m(\cdot)$
and therefore do not require sparsity assumptions on the instance.

\paragraph{Joint confidence intervals from weak pulls.}
After $N$ weak pulls per item, let $\hat v(x)$ be the sample mean and let $r_N(x)$ be a valid $(1-\delta_x)$ half-width for $|\hat v(x)-\weakmean(x)|$ (e.g., empirical Bernstein).
Define per-item confidence intervals
\begin{align*}
	[\CIlo(x),\CIhi(x)] := \bigl[\hat v(x)-r_N(x),\ \hat v(x)+r_N(x)\bigr].
\end{align*}

With $\sum_x \delta_x\le \delta_{\text{weak}}$ (Bonferroni), the joint event
$\Conf:=\bigcap_{x\in\items}\{\val(x)\in[\CIlo(x),\CIhi(x)]\}$
holds with probability at least $1-\delta_{\text{weak}}$.
Unless stated, $\delta_{\text{weak}}=\delta$ (strong exact).

\paragraph{The ambiguous set.}
\label{sec:ambiguous-set}
For $\eta\ge 0$, define the near-tie mass
\[
m(\eta) \;:=\; \bigl|\{x\in D:\ |v(x)-t^\star|\le \eta\}\bigr|.
\]
For each item $x$, let $\CIrad{x}=\tfrac12\big(\CIhi(x)-\CIlo(x)\big)$ denote the confidence radius, and let
$\CImax=\max_{x\in D}\CIrad{x}$ denote the maximum radius across all items.

Given jointly valid weak CIs, we define the ambiguous set
\[
\Amb \;=\; \bigl\{x\in\items:\ \CIlo(x)\le \CIhi_{(k)}\ \text{and}\ \CIhi(x)\ge \CIlo_{(k)}\bigr\},
\]
where $\CIlo_{(k)}$ and $\CIhi_{(k)}$ are the $k$-th largest values in $\{\CIlo(x): x\in\items\}$ and $\{\CIhi(x): x\in\items\}$, respectively. Intuitively, $\Amb$ contains exactly those items whose weak intervals still overlap the uncertain $k$-th boundary band.

\begin{lemma}
	\label[lemma]{lem:ambiguous}
	On the joint event $\Conf$, the ambiguous set satisfies
	$
	|\Amb| \le m(4\CImax).
	$
\end{lemma}

\begin{proof}
	Work on the joint event $\Conf$, where $\val(x)\in[\CIlo(x),\CIhi(x)]$ for all $x$.
	Recall that $\CIrad{x}=\tfrac12(\CIhi(x)-\CIlo(x))$ is the half-width, so $\CIhi(x)-\CIlo(x)=2\CIrad{x}$.
	For any $x\in\Amb$, by definition $\CIlo(x)\le \CIhi_{(k)}$ and $\CIhi(x)\ge \CIlo_{(k)}$.
	Among the $k$ items $x$ with $\val(x)\ge\thresh$, each satisfies on $\Conf$:
	\[
	\CIlo(x) \;\ge\; \val(x) - 2\CIrad{x} \;\ge\; \thresh - 2\CImax.
	\]
	Hence at least $k$ lower bounds are $\ge \thresh-2\CImax$, so $\CIlo_{(k)} \ge \thresh - 2\CImax$.
	Similarly, among the $n-k$ items $x'$ with $\val(x')\le\thresh$ (including the $k$-th item),
	each satisfies:
	\[
	\CIhi(x') \;\le\; \val(x') + 2\CIrad{x'} \;\le\; \thresh + 2\CImax.
	\]
	Thus at least $n-k+1$ upper bounds are $\le \thresh+2\CImax$, implying $\CIhi_{(k)} \le \thresh + 2\CImax$.
	Now, for $x\in\Amb$:
	\begin{itemize}
		\item From $\CIhi(x)\ge \CIlo_{(k)} \ge \thresh - 2\CImax$ and $\val(x)\ge \CIlo(x) = \CIhi(x)-2\CIrad{x}$ follows $\val(x) \;\ge\; \CIhi(x) - 2\CImax \;\ge\; \thresh - 4\CImax$.
		\item From $\CIlo(x)\le \CIhi_{(k)} \le \thresh + 2\CImax$ and $\val(x)\le \CIhi(x) = \CIlo(x)+2\CIrad{x}$ follows $\val(x) \;\le\; \CIlo(x) + 2\CImax \;\le\; \thresh + 4\CImax$.
	\end{itemize}
	Therefore $|\val(x)-\thresh|\le 4\CImax$ for all $x\in\Amb$, hence $|\Amb|\le m(4\CImax)$.
\end{proof}

\section{A Simple One-Shot Approach}
\label{sec:basic-alg}

\Cref{alg:stc} shows a one-shot screen-then-certify (\algbase) algorithm.
The intuition behind STC is simple: items whose intervals lie completely above
(or below) the threshold can be decided immediately and need no further
attention. Only the small set of ``borderline'' items remains ambiguous. By
refining this reduced set with strong evaluations, STC achieves significant
savings in strong calls, with complexity proportional to the size of the
near-tie region rather than the total number of items.
Concretely, STC builds per-item weak CIs $[\CIlo(x),\CIhi(x)]$, forms the threshold band $[\CIlo_{(k)},\CIhi_{(k)}]$, certifies clear IN/OUT items, and queries the strong oracle only for the ambiguous items.

\begin{algorithm2e}[t]
	\caption{\algbase}
	\label{alg:stc}
	\DontPrintSemicolon
	\KwIn{$\items$, $k$, $\delta$, weak pulls per item $N$}
	\KwOut{$(\delta,k)$-PAC top-$k$ set $\widehat T_k$}
	\textbf{Weak intervals:} For each $x\in\items$, query $\weak(x)$ $N$ times and build $[\CIlo(x),\CIhi(x)]$ with joint coverage $1-\delta$.\;
	\textbf{Threshold band:} Let $\CIlo_{(k)}$ be the $k$-th largest $\CIlo(x)$ and $\CIhi_{(k)}$ the $k$-th largest $\CIhi(x)$.\;
	\textbf{Partition:} 
	$IN=\{x:\CIlo(x)>\CIhi_{(k)}\}$, 
	$OUT=\{x:\CIhi(x)<\CIlo_{(k)}\}$,
	$A=\items\setminus(IN\cup OUT)$.\;
	\textbf{Strong refinement:} For all $x\in A$, query $\strong(x)$ and set $\CIlo(x)=\CIhi(x)=\strong(x)$.\;
	\textbf{Output:} $\widehat T_k \leftarrow IN \cup \topkof{A}{k-|IN|}$.\;
\end{algorithm2e}

\begin{theorem}
	\label{thm:identify-refine-threshold}
	If the weak confidence intervals are jointly valid (i.e., on $\Conf$), then \Cref{alg:stc} returns the exact top-$k$ set $T_k^\star$.
	Moreover, it issues at most $m(4\CImax)$ strong oracle calls on $\Conf$ (and hence in expectation up to the failure event of probability at most $\delta$).
\end{theorem}

\begin{proof}
	Work on the joint event $\Conf$.
	If $x\in IN$, then $\CIlo(x)>\CIhi_{(k)}\ge \val_{(k)}=\thresh$, and since $\val(x)\ge \CIlo(x)$ on $\Conf$, we get $\val(x)>\thresh$, hence $x\in\TopK$.
	If $y\in OUT$, then $\CIhi(y)<\CIlo_{(k)}\le \thresh$ and $\val(y)\le \CIhi(y)$, so $\val(y)<\thresh$, hence $y\notin\TopK$.
	The remaining $k-|IN|$ members of $\TopK$ must lie in $A$; Step 4 reveals their exact values, so Step 5 picks exactly those via $\topkof{A}{k-|IN|}$.
	Let $A_0:=\{x:\ \CIhi(x)\ge \CIlo_{(k)} \ \wedge\ \CIlo(x)\le \CIhi_{(k)}\}$ be the initial ambiguous set.
	Intervals only shrink; any $x\notin A_0$ can never become ambiguous, hence never queried.
	Thus \algbase issues at most $|A_0|$ strong calls.
	Lemma~\ref{lem:ambiguous} gives $|A_0|\le m(4\CImax)$ on $\Conf$.
\end{proof}

\begin{theorem}
Let $T_\text{weak}$ and $T_\text{strong}$ be the time complexities of a single weak or strong query, respectively. \Cref{alg:stc} runs in $O(nN \cdot T_\text{weak} + m(4\varepsilon_{\max}) \cdot T_\text{strong})$ time and $O(n)$ space. 
\end{theorem}
\begin{proof}
	The weak phase requires $O(nN \cdot T_\text{weak})$ time to draw $N$ samples per item and construct confidence intervals. The threshold band computation and partitioning take $O(n)$ time using linear-time selection algorithms. The strong refinement queries at most $|A| \le m(4\varepsilon_{\max})$ items, each in $O(T_\text{strong})$ time. In regimes where only $O(k)$ items lie within $4\varepsilon_{\max}$ of the threshold, the strong-refinement time is $O(k\cdot T_\text{strong})$.
\end{proof}

\subsection{Lower Bounds on Strong Oracle Calls}
\label{sec:lower-bounds}
Fix any algorithm whose interaction with the weak oracle induces, at some point in its execution,
a collection of jointly valid intervals $[L(x),U(x)]$ with maximum half-width $\varepsilon_{\max}$
(on an event $\mathcal{E}$ of probability at least $1-\delta_{\text{weak}}$).
Our lower bound is stated in terms of this residual weak uncertainty; it does not depend on whether weak and strong queries are interleaved.

\begin{theorem}
	\label{prop:lb-instance}
	Fix any weak-oracle interaction that yields jointly valid intervals with radius $\varepsilon_{\max}$ on $\mathcal{E}$.
	For any $(\delta,k)$-PAC algorithm, there exist instances consistent with these intervals on which the algorithm must make at least $\Omega\big(m(\CImax)\big)$ strong oracle calls.
\end{theorem}
\begin{proof}
	Fix the weak-oracle interaction yielding jointly valid intervals with maximum half-width $\CImax$ on the event $\Conf$.
	Choose any $\xi\in(0,1)$ such that $[\xi-2\CImax,\ \xi+2\CImax]\subset(0,1)$ and $\xi-\gap-2\CImax>0$.
	Fix an integer $m$ and choose a subset $S\subseteq\items$ with $|S|=m$.
	For each subset $T\subseteq S$ with $|T|=k$, define the instance $\val_T$ by
	\[
	\val_T(x)=
	\begin{cases}
		\xi+\CImax/2, & x\in T,\\
		\xi-\CImax/2, & x\in S\setminus T,\\
		\xi-\gap,     & x\in \items\setminus S.
	\end{cases}
	\]
	Choose $\gap>2\CImax$ in this construction, so the values $\xi\pm \CImax/2$ are well separated from $\xi-\gap$.
	We now exhibit jointly valid weak intervals consistent with this weak budget:
	
	\noindent$\bullet$ For $x\in S$: set $[\CIlo(x),\CIhi(x)]=[\xi-\CImax,\ \xi+\CImax]$.
	These intervals 
	contain both possibilities $\xi\pm\CImax/2$, so every $x\in S$ remains ambiguous.
	
	\noindent$\bullet$ For $x\notin S$: set $[\CIlo(x),\CIhi(x)]=[\xi-\gap-\tfrac{\CImax}{2},\ \xi-\gap+\tfrac{\CImax}{2}]$.
	These have half-width $\CImax/2\le\CImax$ and lie strictly below $\xi-\tfrac{3}{2}\CImax$ because $\gap>2\CImax$.
	
	With these intervals, at least $k$ lower bounds among $S$ are $\ge \xi-\CImax$, hence $L_{(k)}\ge \xi-\CImax$.
	For $x\notin S$, we have $U(x)\le \xi-\gap+\CImax/2\le \xi-\tfrac{3}{2}\CImax < \xi-\CImax \le L_{(k)}$, 
	so all $x\notin S$ are weak-certified OUT.
	
	Moreover, every $x\in S$ has the same interval $[\xi-\CImax,\xi+\CImax]$, 
	and every $x\notin S$ is already OUT. 
	Consider any algorithm that strongly queries only a strict subset $Q\subsetneq S$.
	Then there exist $T\neq T'$ with $T\cap Q=T'\cap Q$.
	The algorithm observes the same weak evidence and the same strong reveals 
	on $Q$, therefore it outputs the same $\widehat T_k$ on $\val_T$ and $\val_{T'}$,
	which must be wrong on at least one of them.
	Thus any $(\delta,k)$-PAC algorithm must issue $\Omega(|S|)=\Omega(m(\CImax))$ strong calls in the worst case. By construction, this instance satisfies $m(\varepsilon_{\max})\ge m$.
\end{proof}

\section{Adaptive Strong Oracle Certification}
\label{sec:adaptive}

While \algbase is worst-case optimal, it refines \emph{all} ambiguous items.
ACE begins, like \algbase, by constructing initial confidence intervals $[L(x), U(x)]$ for all items using a fixed number of weak oracle queries per item, $N$. It then enters an adaptive refinement loop. Instead of refining a large ambiguous set, \ACE iteratively identifies the two most \emph{critical} items:
the worst IN (smallest $\CIlo$ in the tentative top-$k$) and the best OUT (largest $\CIhi$ outside). It stops once $\CIlo(\text{worst IN}) \ge \CIhi(\text{best OUT})$.
These two items are the ``bottleneck'' preventing certification. By collapsing the
interval of whichever is more uncertain, ACE guarantees monotone progress and
terminates as soon as the worst-in item is safely above the best-out item. This
selection rule ensures that every strong call directly contributes to shrinking
the certification margin.

\begin{algorithm2e}[t]
	\caption{\ACE}
	\label{alg:ace}
	\DontPrintSemicolon
	\KwIn{$\items$, $k$, $\delta$, weak pulls per item $N$}
	\KwOut{$(\delta,k)$-PAC top-$k$ set $\widehat T_k$}
\textbf{Weak intervals:} For each $x\in\items$, query $\weak(x)$ $N$ times and build $[\CIlo(x),\CIhi(x)]$ with joint coverage $1-\delta$.\;
	
	\textbf{Adaptive strong query phase:}\;
	
	\While{true}{
		$S_k \leftarrow$ set of $k$ items with largest \emph{upper} bounds $\CIhi(x)$ \tcp*{optimistic top-$k$}
		$i \leftarrow \arg\min_{x\in S_k} \CIlo(x)$ \tcp*{critical IN}
		$j \leftarrow \arg\max_{y\notin S_k} \CIhi(y)$ \tcp*{critical OUT}
		\If{$\CIlo(i)\ge \CIhi(j)$}{\textbf{return} $S_k$ \tcp*{certified set $\widehat T_k$}}\label{alg:ace:break}
		$x^\star \leftarrow \arg\max\big\{\CIhi(i)-\CIlo(i),\ \CIhi(j)-\CIlo(j)\big\}$\;
		query strong: $v^\star\leftarrow \strong(x^\star)$\;
		set $\CIlo(x^\star)=\CIhi(x^\star)=v^\star$\;\label{alg:ace:collabs}
	}
\end{algorithm2e}

\begin{theorem}
	\label{thm:pac-lucb}
	If the weak confidence intervals are jointly valid (i.e., on $\Conf$), then \Cref{alg:ace} returns the exact top-$k$ set $T_k^\star$.
	Moreover, it issues at most $m(4\CImax)$ strong oracle calls on $\Conf$ (and hence fails with probability at most $\delta$ overall).
\end{theorem}

\begin{proof}
Each iteration of the while loop either terminates (line \ref{alg:ace:break}) or collapses one interval to a point (line \ref{alg:ace:collabs}); hence each strongly queried item is queried at most once.
Let $A_0 := \{x\in\items:\ \CIhi(x)\ge \CIlo_{(k)} \ \wedge\ \CIlo(x)\le \CIhi_{(k)}\}$ denote the ambiguous set induced by the initial weak intervals.
	For correctness, work on the joint event $\Conf$. When the algorithm terminates, $\CIlo(i)\ge \CIhi(j)$. For any $x\in S_k$ and $y\notin S_k$, we have
	$
	\val(x)\ge \CIlo(x)\ge \CIlo(i)\ge \CIhi(j)\ge \CIhi(y)\ge \val(y).
	$
	Therefore, all items in $S_k$ have higher values than all items outside $S_k$, so $S_k=\TopK$.
	Finally, by Lemma~\ref{lem:ambiguous} we have $|A_0|\le m(4\CImax)$ on $\Conf$.
	Since each iteration collapses one interval, \ACE performs at most $|A_0|$ strong queries on $\Conf$.
	Unconditioning $\Conf$ accounts for failure probability at most $\delta$.
\end{proof}

\begin{theorem}
Let $T_\text{weak}$ and $T_\text{strong}$ be the time complexities of a single weak or strong query, respectively. 
\Cref{alg:ace} runs in $O(nN\cdot T_\text{weak} + m(4\CImax) \cdot \log n + m(4\CImax)\cdot T_\text{strong})$ time using at most $O(n)$ space.
\end{theorem}
\begin{proof}
The initial weak phase requires $O(nN\cdot T_\text{weak})$ time. Each iteration of the adaptive loop requires $O(\log n)$ time using an order-statistic tree (e.g., augmented red-black tree) to identify the optimistic top-$k$ set $S_k$ and find the critical items $i$ and $j$, plus $O(T_\text{strong})$ time for the strong query. At most $m(4\CImax)$ iterations are needed.
Space complexity is $O(n)$ to store the per-item confidence intervals.
\end{proof}

While $\ACE$ can require many strong calls on instances with large near-tie mass, its practical performance is governed by the specific values of the items and the precision of the weak oracle. The number of strong oracle queries directly depends on the number of items whose confidence intervals are close to the decision threshold. We characterize \ACE's complexity in terms of the near-tie mass, establishing matching upper and lower bounds up to a constant factor in the radius parameter.

\begin{theorem}
	\label{thm:instance-dep-tight}
	Let the initial jointly valid intervals have maximum radius $\CImax:=\max_x \CIrad{x}$.
	Any $(\delta,k)$-PAC algorithm must, on some instance consistent with these intervals, make $\Omega\big(\mnearties{\CImax}\big)$ strong calls, and 
	\ACE\ terminates after at most $O\big(\mnearties{4\CImax}\big)$ strong calls.
\end{theorem}

\begin{proof}
	On $\Conf$, any ambiguous $x$ must satisfy $|\val(x)-\thresh|\le 4\CImax$; hence $|A_0|\le \mnearties{4\CImax}$.
	\ACE\ only queries items in $A_0$ and never more than once, so its strong calls are $\le \mnearties{4\CImax}$.

For the lower bound, we construct a packing of instances with $m:=\mnearties{\CImax}$ items within $\CImax$ of $\thresh$ whose weak intervals remain indistinguishable under the capped budget (as in the proof of \Cref{prop:lb-instance}).
	Any algorithm that does not query a linear fraction of these $m$ items will output the same set on two instances that differ by flipping one such item across $\thresh$, hence err on one of them.
	Thus $\Omega(m)$ strong calls are necessary.
\end{proof}

\section{The Fully Adaptive Algorithm}
\label{sec:full-adaptive}

\ACE adaptively selects items for \emph{strong} evaluation, but its efficiency depends on the quality of the \emph{weak} confidence intervals obtained from a non-adaptive, fixed budget $N$ per item. We improve this by making the weak phase adaptive as well. The resulting \ACEW\ algorithm has two sequential adaptive phases.

\paragraph{Phase I: Adaptive Weak Allocation (AWA).}
Instead of fixing a uniform per-item budget $N$, we allocate a total weak budget $B$ adaptively. 
In each iteration, AWA forms the current ambiguous set $\Amb$ and selects 
$
x^\star \in \arg\max_{x \in A} \big(U(x)-L(x)\big),
$
the ambiguous item with the widest confidence interval. 
The next weak pull is then allocated to $x^\star$, tightening its anytime-valid CI. 
This strategy concentrates the weak oracle effort on the items that block certification, 
shrinks the maximum ambiguous radius 
$
\varepsilon'_{\max} = \max_{x \in A} \tfrac{1}{2}\big(U(x)-L(x)\big),
$ 
leading to tighter intervals for the ambiguous items. 
To maintain PAC guarantees in Phase~I, we use time-uniform confidence sequences
that are valid for bounded observations.
Since $\tilde v(x)\in[0,1]$ (after normalization), we construct an empirical-Bernstein
confidence sequence $r^{\text{any}}_w(x)$ such that
\[
\Prb\Big(\forall w\ge 1:\ |\hat v_w(x)-\mu(x)|\le r^{\text{any}}_w(x)\Big)
\;\ge\; 1-\tfrac{\delta_{\text{weak}}}{n}.
\]
A union bound over $x\in D$ yields joint time-uniform coverage
$1-\delta_{\text{weak}}$ for all items.

We warm-start the first phase with a minimum of weak queries ($\mmin$) for each item, and we bound the maximum of weak queries ($\mmax$) for each item to avoid spending excessive amounts of budget on pathological items. 

\begin{algorithm2e}[htb]
	\caption{\ACEW}
	\label{alg:acew}
	\DontPrintSemicolon
	\KwIn{$\items$, $k$, total weak budget $B$,  $\delta$,  $\mmin$, and  $\mmax$}
	\KwOut{$(\delta,k)$-PAC top-$k$ set $\widehat T_k$}
	\textbf{Phase I (AWA):} Initialize per-item weak counts and anytime-valid CIs; give each $x\in\items$ a warm-start of $\mmin$ weak pulls. Set $B\gets B - n\cdot \mmin$\;
	
	\While{$B>0$}{
		
		$\Amb \leftarrow$ identify current ambiguous set \;
		\lIf{$\Amb$ is empty}{\textbf{break}}

		$x^\star \leftarrow \argmax_{x \in A \;:\; w(x) < w_{\max}} (U(x)-L(x))$
		Draw one weak sample on $x^\star$, update its anytime-valid CI, and set $B \gets B-1$.
	}
	Freeze the weak CIs for Phase II. \; 
	\textbf{Phase II (ACE):} Run Alg.~\ref{alg:ace} initialized with  $[\CIlo,\CIhi]$.\;
	\Return $\widehat T_k$.
\end{algorithm2e}

\paragraph{Phase II: Adaptive Strong Certification.}
Given the frozen weak intervals $\{[\CIlo(x),\CIhi(x)]\}$, Phase~II is identical to \ACE: it adaptively queries the strong oracle on the remaining ambiguous items until the final top-$k$ is certified.

\begin{theorem}
	\label{thm:awa-ace-complexity}
	Let $\{[\CIlo(x),\CIhi(x)]\}$ be the intervals at the end of Phase~I
	with
	$\CImax' := \max_{x\in A} \tfrac12\big(\CIhi(x)-\CIlo(x)\big)$. 
	Then Phase~II (\ACE) issues at most $O\big(\mnearties{4\CImax'}\big)$ strong oracle calls. 
\end{theorem}

\begin{theorem}
\Cref{alg:acew} runs in $O(B\cdot T_\text{weak} + B \log n + m(4\CImax') \cdot \log n + m(4\CImax')\cdot T_\text{strong})$ time using at most $O(n)$ space , where $\CImax'$ is the maximum radius after Phase I.
\end{theorem}
\begin{proof}
	Phase I performs $B$ weak queries, each taking $O(T_\text{weak})$ time, with $O(\log n)$ overhead per query to identify the widest interval in the ambiguous set using a max-heap keyed by interval width. The ambiguous set $A$ is maintained incrementally; since intervals only shrink during Phase I, items can only leave $A$, never enter. Phase II runs \ACE on the resulting intervals, requiring $O(m(4\CImax') \cdot (\log n + T_\text{strong}))$ time using an order-statistic tree as described above. Space complexity is $O(n)$ to maintain per-item statistics and CIs.%
\end{proof}

\section{Discussion}

Our conditional lower bound scales as $\Omega\big(m(\CImax)\big)$, while our upper bounds scale as $O\big(m(4\CImax)\big)$.
Thus there remains a constant-factor gap of 4 in the radius parameter inside $m(\cdot)$. This does not necessarily translate to a factor-of-4 gap in the number of strong calls because the relationship between $m(\CImax)$ and $m(4\CImax)$ is instance-dependent.
Closing this gap
appears to require either a different certification mechanism or additional assumptions on the instance (e.g., regularity of the value distribution near $\thresh$); we leave this as an open question.

Furthermore, our framework's performance and guarantees are subject to the following conditions.
The strong-call complexity is fundamentally tied to the ``near-tie mass,'' $\mnearties{\cdot}$. Instances where many item values cluster near the top-$k$ threshold $\thresh$ are inherently difficult, and the cost for any valid algorithm will be proportionally high. Moreover, the practical benefit of our adaptive methods depends on the quality of the weak oracle. If weak pulls are excessively noisy (leading to a large $\CImax$), the initial confidence intervals will be too wide to effectively prune the search space, and the strong-call savings over the \algbase baseline will be minimal. 
Finally, the presented algorithms are specifically designed for identifying a top-$k$ set of a fixed, pre-specified size $k$. Alternative ranking problems, such as value-based thresholding (e.g., finding all items with $\val(x) > \tau$), are not directly addressed and would require a separate treatment.

Moreover, our strong-call bounds are stated \emph{conditional} on the realized weak uncertainty
$\CImax=\max_{x\in\items}\tfrac12\big(\CIhi(x)-\CIlo(x)\big)$ after the weak phase.
Thus, selecting a uniform weak budget $N$ can be nontrivial when the separation around the top-$k$ boundary is unknown.
In general, any valid per-item CI construction implies a (problem-dependent) relationship between $N$ and $\CImax$.
Rather than requiring users to estimate an unknown gap parameter, a practical approach is to choose $N$ by directly monitoring the weak-phase ambiguity:
compute the ambiguous set $\Amb$ induced by the weak CIs and increase $N$ (e.g., double $N$) until $|\Amb|$ fits within the available strong budget.
This requires no knowledge of an explicit gap and directly targets the quantity that drives strong cost
(since on the joint CI event, both \algbase and \ACE make at most $|\Amb|$ strong calls).
For a fixed total weak budget, \ACEW provides an alternative by adaptively concentrating weak pulls on ambiguous items, often reducing
the post-allocation radius $\CImax'$ and hence the required number of strong calls.

\section{Experiments}\label{sec:experiments}
We evaluate our methods on synthetic data to validate theoretical scaling, and on a real-world data valuation task to demonstrate practical utility.
We compare \ACE and \ACEW to the \algbase baseline, the natural non-adaptive method in the two-oracle PAC top-$k$ certification setting.
We additionally introduce a thresholding algorithm (\emph{TA-Certify}) that first ranks items by the weak estimates, then queries the strong oracle sequentially in that order.
Crucially, it uses the {same jointly valid weak confidence intervals} $[L(x),U(x)]$ as our methods to derive an early-stopping certificate:
after each strong evaluation, it maintains the current $k$-th largest \emph{verified} strong value $t_S$, and stops as soon as the worst verified top-$k$ item is guaranteed (under the weak CIs) to exceed any remaining unverified item.\looseness=-1

We implemented all algorithms in Python 3.11.
All experiments run on a consumer laptop with 13th Gen Intel(R) Core(TM) i5-1335U @ 4.6 GHz and 16~GB of RAM.
The source code is provided in the supplementary materials.

\begin{figure}[htbp]
	\centering
	\subfigure[Scalability in $n$]{%
		\label{fig:scalability-n}%
		\includegraphics[width=0.3\linewidth]{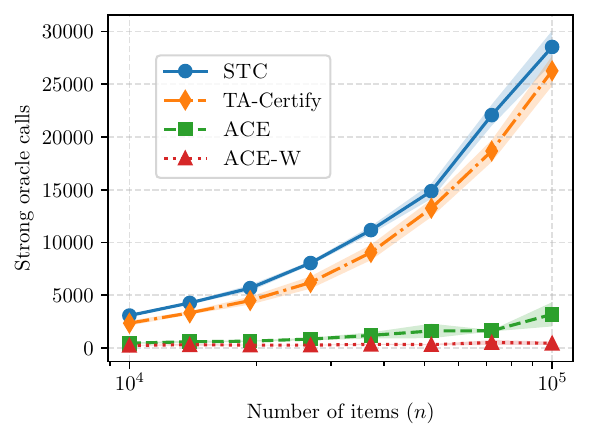}%
	}\hfill
	\subfigure[Scalability in $k$]{%
		\label{fig:scalability-k}%
		\includegraphics[width=0.3\linewidth]{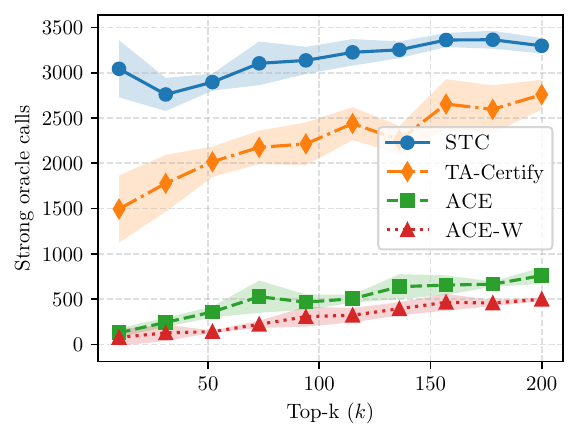}%
	}\hfill
	\subfigure[Varying difficulty ratio.]{%
		\label{fig:hardness}%
		\includegraphics[width=0.3\linewidth]{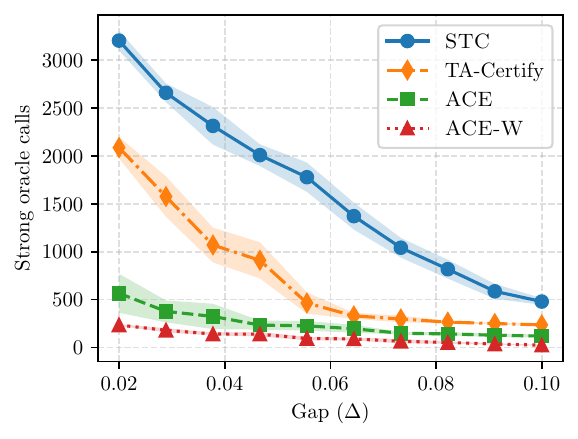}%
	}
	\caption{Performance on synthetic data. Our adaptive \ACE and \ACEW algorithms consistently outperform the non-adaptive baselines.}
	\label{fig:synthetic}
\end{figure}

\smallskip
\noindent
\textbf{Scalability and Robustness.}
We generate values with a clear gap $\Delta$ around the decision threshold $t^\star=0.5$, with $k$ top items, up to $2 k$ near-tie items in $[t^\star-\eta,t^\star+\eta]$ ($\eta=0.5\Delta$), and the remainder well below threshold.   
The \emph{strong oracle} returns the exact value $v(x)$.
The \emph{weak oracle} aligns with our theoretical model from \Cref{sec:setting}, returning noisy observations. Specifically, a query for item $x$ yields $\tilde v(x) = v(x) + \xi$, where $\xi \sim \mathcal{N}(0, \sigma^2)$ is random noise. We report the number of strong calls averaged over 10 runs with 95\% CIs.
Across all experiments, all methods always recover the exact top-$k$ set. 
We use $n=10^4$, $k=100$, $\Delta=0.05$, $\sigma=0.1$, $N=12$, $B=n\cdot N$, $\mmin=6$, and $\mmax=B$ unless stated otherwise.

\Cref{fig:synthetic} validates the theoretical scaling of our algorithms.
\emph{(a) Scalability in $n$:} \Cref{fig:synthetic}a demonstrates the critical dependency on the ambiguous set size $|A|$. As $n$ increases under fixed item density, the number of items falling into the near-tie region grows linearly ($|A| \approx 1.8 \times 10^3$ at $n=10^4$ to $\approx 5.3 \times 10^4$ at $n=3 \times 10^5$). 
Since STC must verify every item in $A$, its cost explodes linearly with $n$. 
In contrast, ACE decouples strong queries from the total ambiguous mass, growing only sub-linearly ($\approx 200 \to 800$ calls). 
Notably, \ACEW (red) achieves effectively constant cost ($\approx 100$ calls) regardless of $n$, proving that adaptively targeting weak budget to boundary items prevents the accumulation of false positives that plague the non-adaptive methods.

\emph{(b) Scaling with $k$:} While all methods scale linearly with the target set size $k$ (\Cref{fig:synthetic}b), \ACE and \ACEW exhibit a significantly flatter slope, indicating superior per-item certification efficiency.

\emph{(c) Robustness to Hardness:} As the gap $\Delta$ narrows (\Cref{fig:synthetic}c), the problem becomes harder and costs rise for all methods. However, \ACE maintains a large efficiency gap over the baselines, avoiding the sharp cost spike seen in STC.

\smallskip
\noindent
\textbf{Empirical tightness of bounds.}
We measure the tightness ratio $\rho = \text{(strong calls)} / m(\CImax)$, where $m(\CImax) = |\{x : |\val(x) - \thresh| \le \CImax\}|$ (computed post-hoc). Across all settings, \ACE achieves $\rho = 2.56\%$ on average while \ACEW achieves $1.31\%$. At $n = 3.16 \times 10^5$, both algorithms use $< 0.3\%$ of the lower bound (90.7 calls vs. $m(\CImax) \approx 2.2 \times 10^5$ for \ACEW).

This empirical gap is dramatically smaller than the theoretical worst-case factor of 4 in the radius parameter. While the upper bound guarantees $O(m(4\CImax))$ calls, practical performance is $0.01$--$0.06 \cdot m(\CImax)$. This supports our hypothesis that the constant factor arises from worst-case compounding of uncertainties, and that real-world value distributions rarely exhibit the pathological alignment required to trigger the worst case. The adaptive weak allocation phase provides consistent $\approx 2\times$ improvement over uniform sampling across all settings.

\smallskip
\noindent
\textbf{Real-World Data Valuation.}
We evaluate our framework on a data valuation task for text classification. 
The objective is to efficiently identify the $k$ most valuable training documents for a machine learning model. 
We use the \emph{20 Newsgroups} dataset~\citep{lang1995newsweeder}, focusing on documents from two categories: \texttt{comp.graphics} and \texttt{sci.med}. 
The model trains a $k$-Nearest Neighbors (kNN) classifier to distinguish between these topics. 
A document's value is defined by its contribution to the classifier's accuracy, quantified via the \emph{Shapley value}~\citep{ghorbani2019data}. 
Exact computation is prohibitively expensive, motivating our two-oracle framework. 
Both oracles rely on Monte Carlo (MC) estimation using the same validation data. 
The \emph{weak oracle} uses $N=2^6$ MC rounds, providing fast but high-variance estimates. 
The \emph{strong oracle} uses $2^{13}$ MC rounds, yielding low-variance estimates treated as ground truth.
We set $\delta=0.05$. For \ACEW, we use a total weak budget $B=N\cdot n$, with $\mmin=16$ and $\mmax=128$.
We consider the goal of identifying the top-$10$ most valuable documents out of $n=100$.

\Cref{table:results} reports the performance metrics (mean $\pm$ std) over 8 independent runs.
\algbase\ is the most expensive, requiring an average of 83.5 strong calls to certify the top-10. 
TA improves upon this by exploiting sorted access, reducing calls to 66.0.
\ACE\ significantly outperforms both, cutting strong calls to 33.3 (a reduction of $\approx$50\% over TA and $\approx$60\% over \algbase).
\ACEW\ further reduces strong calls to 29.4 and shrinks the ambiguous set size ($|A|$) by $\approx$8\% compared to the uniform weak allocation.
Because the strong oracle is computationally intensive, these reductions in query count translate to massive wall-clock savings. 
While TA provides a $1.2\times$ speedup over \algbase, \ACE\ and \ACEW\ achieve speedups of $2.4\times$ and $2.8\times$, respectively.

\begin{table}[t]
	\centering
	\caption{{Data Valuation Results (Top-10 Identification).} 
		Performance metrics on the 20 Newsgroups dataset (mean $\pm$ std over 10 runs). 
	}
	\label{table:results}
	\resizebox{0.8\textwidth}{!}{
		\begin{tabular}{lcccc}
			\toprule
			\textbf{Algorithm} & \textbf{Strong Calls} & \textbf{Ambiguity ($|A|$)} & \textbf{Total Time (s)} & \textbf{Speedup vs STC} \\
			\midrule
			\algbase (STC) & $83.5 \pm 19.1$ & $83.5 \pm 19.1$ & $10{,}907 \pm 2{,}837$ & $1.0\times$ \\
			TA-Certify & $66.0 \pm 26.5$ & $83.5 \pm 19.1$ & $8{,}768 \pm 3{,}527$ & $1.2\times$ \\
			\textbf{ACE (Ours)} & $\mathbf{33.3} \pm 9.3$ & $83.5 \pm 19.1$ & $4{,}575 \pm 1{,}349$ & $\mathbf{2.4\times}$ \\
			\textbf{ACE-W (Ours)} & $\mathbf{29.4} \pm 11.4$ & $\mathbf{76.9} \pm 19.5$ & $\mathbf{3{,}932} \pm 1{,}593$ & $\mathbf{2.8\times}$ \\
			\bottomrule
		\end{tabular}
	}
\end{table}

\section{Conclusion and Future Work}

We introduced a two-oracle framework for certifying the exact top-$k$ set under PAC guarantees. Our algorithms \ACE\ and \ACEW\ adaptively focus strong evaluations on critical items, yielding instance-dependent complexity governed by the near-tie mass and substantially reducing strong oracle usage in experiments.

Future directions include: (i) closing the constant-factor gap between our $O(m(4\CImax))$ upper bound and $\Omega(m(\CImax))$ lower bound, either through tighter analysis or additional structural assumptions; (ii) fully adaptive weak-phase stopping using anytime-valid confidence sequences to handle unknown separation margins while maintaining joint coverage guarantees.

\end{document}